\documentclass[twoside]{article}

\usepackage[accepted]{aistats2025}
\usepackage{amssymb,amsfonts,amsmath,amsthm} 
\usepackage{enumerate,enumitem,tikz,graphicx,mathrsfs,eucal,verbatim, bbm, derivative}
\usepackage{hyperref}
\usepackage{tkz-graph, tikz-cd, svg}
\usepackage{graphicx}
\usepackage{chessfss}
 \usepackage{relsize}
\usepackage[makeroom]{cancel}
\usepackage{adjustbox}
\usepackage{diagbox}
\usepackage[thinc]{esdiff}

\usepackage{caption}
\usepackage{subcaption}

\let\svthefootnote\thefootnote
\newcommand\freefootnote[1]{%
  \let\thefootnote\relax%
  \footnotetext{\hspace{-1.5em}#1}%
  \let\thefootnote\svthefootnote%
}

 \usepackage[
    backend=biber,
    style=authoryear,
    maxcitenames=1,
    uniquelist=false
  ]{biblatex}
\theoremstyle{plain}
\newtheorem{thm}{Theorem}[section]

\newtheorem{infthm}[thm]{Informal Theorem}

\newtheorem{prop}[thm]{Proposition}

\newtheorem{lem}[thm]{Lemma}
\newtheorem{cor}[thm]{Corollary}
\theoremstyle{definition}
\newtheorem{defn}{Definition}

\theoremstyle{remark}
\newtheorem{rmk}{Remark}[section]

\newcommand{\symm}[2]{\textnormal{Sym}_{#1}(\mathbb{R}^{#2})}
\newcommand{\symmm}[2]{\textnormal{Sym}_{#1}(#2)}
\newcommand{\edd}{\textnormal{gED}}
\newcommand{\conv}[1]{\star_{#1}}

\newcommand{\vs}[1]{\textcolor{blue}{{#1}}}

\begin{document}

%
\runningtitle{Polynomial Convolutional Networks}

%

\twocolumn[

\aistatstitle{On the Geometry and Optimization \\of Polynomial Convolutional Networks}

\aistatsauthor{Vahid Shahverdi * \And Giovanni Luca Marchetti * \And Kathl\'en Kohn *}

\aistatsaddress{KTH Royal Institute of Technology, Stockholm, Sweden} ]

\begin{abstract}
We study convolutional neural networks with monomial activation functions. Specifically, we prove that their parameterization map is regular and is an isomorphism almost everywhere, up to rescaling the filters. By leveraging on tools from algebraic geometry, we explore the geometric properties of the image in function space of this map -- typically referred to as neuromanifold. In particular, we compute the dimension and the degree of the neuromanifold, which measure the expressivity of the model, and describe its singularities. Moreover, for a generic large dataset, we derive an explicit formula that quantifies the number of critical points arising in the optimization of a regression loss.

\end{abstract}

\section{INTRODUCTION}\label{sec:intro}
Deep neural networks -- and, in general, parametric machine learning models -- define a space of functions as their parameters vary. These spaces are often referred to as \emph{neuromanifolds} \parencite{marchetti2025invitation, kohn2024geometry, calin2020neuromanifolds}. Understanding the geometry of neuromanifolds is a subtle yet fundamental challenge due to its intimate connection to the training process. Namely, neural networks learn by following a gradient flow that attracts the model to (an estimate of) the ground-truth function, which can be interpreted as minimizing a functional distance over the neuromanifold. Therefore, geometric problems over neuromanifolds -- such as nearest point problems -- are related to the learning dynamics of the corresponding models. 

\begin{figure}[t!]
    \centering
    \includegraphics[width=1.\linewidth]{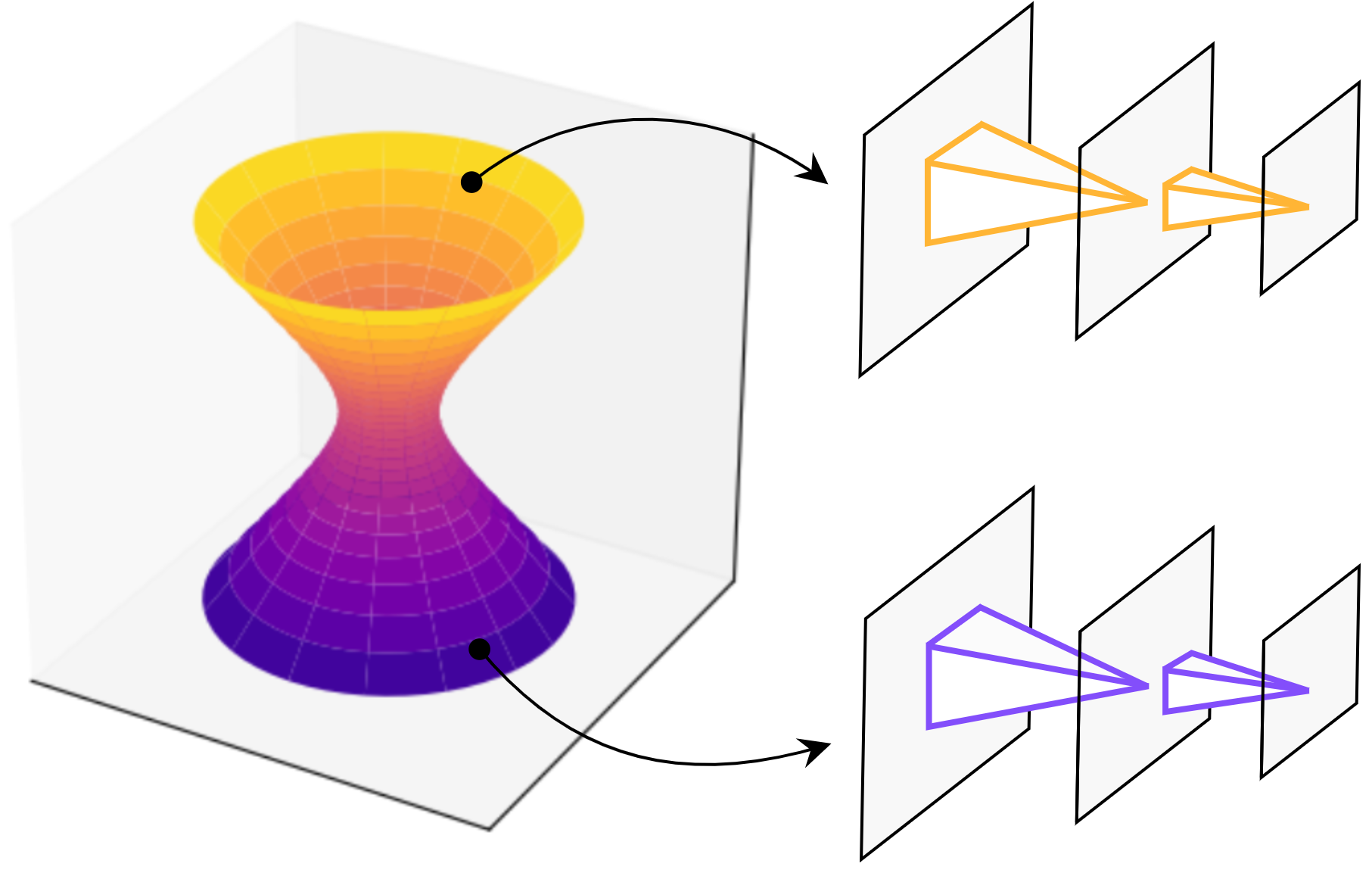}
    \caption{Illustration of a Segre--Veronese variety parametrizing CNNs.}
    \label{fig:segre}
\end{figure}

For neural networks with polynomial activation functions, the neuromanifold is a (semi-) algebraic set, i.e., it is defined by polynomial equalities and inequalities. This enables to exploit tools from the rich field of \emph{algebraic geometry} for analyzing neuromanifolds. In particular, various algebro-geometric invariants can provide insights into fundamental machine learning aspects. \freefootnote{\quad *Equal contribution.}For example, the \emph{dimension} of the neuromanifold  plays a central role in statistical learning theory\footnote{In this context, neuromanifolds are referred to as `hypothesis spaces', and are usually considered in a combinatorial version.} since, according to the Fundamental Theorem of Learning \parencite{shalev2014understanding}, it controls the sample complexity of learnability. Moreover, the \emph{degree} is a curvature-based invariant that controls the concentration of measure around the neuromanifold and, in particular, the approximation error of the given model \parencite{basu2023hausdorff}. Dimension and degree are, in other words, basic measures of expressivity. Lastly, the \emph{Euclidean distance degree} \parencite{draisma2016euclidean} is a modern invariant counting the singularities of the distance function over the neuromanifold from an external point. In particular, it upper-bounds the number of (locally-) nearest points. Since these invariants are well-understood for a wide class of algebraic varieties, they can provide insights into the learning process of networks of various architectures.

In this work, we consider deep \emph{Convolutional} Neural Networks (CNNs --  \cite{fukushima1979neural, lecun1995convolutional}) with monomial activation functions, and study their neuromanifolds. Our motivation is twofold. First, CNNs are popular models deployed in various signal processing domains. Historically, they have played a central role in several modern breakthroughs across deep learning \parencite{krizhevsky2012imagenet, van2016wavenet}, and have recently seen a resurgence in computer vision \parencite{liu2022convnet}. Second, the convolutional architecture is particularly suitable for algebraic geometry. Indeed, the convolution is a well-behaved algebraic operation, which translates into geometric properties for the associated neuromanifold and its parametrization. 

\subsection{Summary of Results}\label{sec:summres}
We provide theoretical results on both the geometry and the optimization of polynomial CNNs. Our arguments involve tools and ideas from algebraic geometry. In order to analyze neuromanifolds of polynomial CNNs, we study their parametrization. In this context, our main result states, informally, the following. 
\begin{infthm}[Section \ref{sec:geome}]\label{thm:inf1}
Up to rescaling each filter, the parameterization of the neuromanifold of a polynomial CNN is regular and is an isomorphism almost everywhere. The remaining fibers are finite. 
\end{infthm}
Intuitively, this shows that CNNs are parametrized in an `optimal way' since, once the scaling symmetries are factored out, the parameters are in a one-to-one regular correspondence (almost everywhere) with the associated functions. This is in contrast with other neural architectures -- such as fully-connected networks -- where the parametrization exhibits large fibers and critical points, translating into redundancy in the parameter space \parencite{kileel2019expressive}. 

An immediate consequence of Informal Theorem \ref{thm:inf1} is an expression for the dimension and the degree of the neuromanifold -- see Corollary \ref{cor:finiteparam}. As mentioned in Section \ref{sec:intro}, these two invaraints play a central role in learning theory. In particular, we observe that the dimension grows linearly w.r.t. the number of layers, while the degree grows (super-) exponentially. This provides an intuitive explanation to the importance of depth in neural networks; neuromanifolds of deeper CNNs remain low-dimensional -- translating into a low sample complexity of learning -- while efficiently filling their ambient function space -- translating into expressive power.  

Informal Theorem \ref{thm:inf1} implies other geometric properties. First, it follows that the neuromanifold is closed in the Zariski topology of its ambient space (see Proposition \ref{prop:zarboundary}), and therefore in the Euclidean one as well. In terms of learning dynamics, this shows that the neuromanifold contains its asymptotic limits. Another consequence is the description of \emph{singularities} of the neuromanifold, i.e., special CNNs where the neuromanifold looks degenerate. Studying singularities of neuromanifolds and their relation to the learning dynamics is the focus of Singular Learning Theory \parencite{watanabe2009algebraic,amari2006singularities}. It follows from Informal Theorem \ref{thm:inf1} that the singular points are the simplest possible -- specifically `nodal' singularities -- which, as discussed below, translates into advantages for the optimization dynamics. Such singularities can only come from `subnetworks', meaning that they arise when (specific) weights vanish -- see Corollary \ref{cor:singpts}. The relation between singularities and subnetworks might be interpreted as a geometric explanation of the implicit bias observed in neural networks to converge towards simpler architectures during training \parencite{frankle2018lottery}.

In order to prove Informal Theorem \ref{thm:inf1}, we first remove the scaling symmetries by working in the projective space, and then to rely on tools from projective geometry. In particular, we leverage on birational geometry which, roughly speaking, concerns with properties that are preserved almost everywhere. Moreover, we factorize the projectified parametrization into a linear projection and the \emph{Segre--Veronese} embedding. The latter is a classical map in algebraic geometry whose properties are well-understood -- see Figure \ref{fig:segre} for an illustration. 

As anticipated before, understanding the geometry of the neuromanifold enables us to study the learning dynamics of polynomial CNNs. In particular,  considering the square-error regression loss, we prove the following. 
\begin{infthm}[Section \ref{sec:optim}]
For large generic datasets, the square-error loss can be rephrased as a distance minimization over the neuromanifold from an external polynomial function. There is an explicit upper bound for the number of (complex) critical points of the loss. 
\end{infthm}
The formula above is obtained via the theory of the Euclidean distance degree, which has a known expression for the Segre--Veronese variety.
Our formula upper-bounds the number of both real and complex critical points. Since local minima are critical points, we obtain an upper bound for their number as well. Finally, we argue that  counting critical points over the neuromanifold is (almost) equivalent to counting them, up to scaling, over the parameter space, which is where optimization is performed in practice. Specifically, since the parametrization is regular (Theorem \ref{thm:inf1}), the dynamics are equivalent when considered in parameter space or over the neuromanifold. Lastly, singular points of the neuromanifold require special care; to this end, we show that they are not critical for the distance function (except for the 0-function), which is possible via our description of singularities. 

\section{RELATED WORK}
Since our work is concerned with neuromanifolds of polynomial CNNs, we review the literature around neuromanifolds and polynomial neural networks. 

\paragraph{Algebraic Geometry of Neuromanifolds.} Neuromanifolds have been analyzed via algebraic geometry in several instances. Fully-connected polynomial networks have been discussed in \parencite{kileel2019expressive, finkel2024activation}, with a focus on problems regarding dimensionality, while linear CNNs have been discussed in \parencite{kohn2022geometry, kohn2023function, shahverdi2024algebraic}, with a focus on singularities and critical points of the loss function. More recently, neuromanifolds of linear self-attention networks have been considered \parencite{henry2024geometrylightningselfattentionidentifiability}. Here, we contribute to this line of research by discussing the geometry of neuromanifolds defined by polynomial CNNs. While the previous literature focuses mainly on the linear case -- with only partial results for the general polynomial one -- our work is the first to provide a comprehensive description of both the parametrization and the neuromanifold of a polynomial model. 
\paragraph{Polynomial Activation Functions.} Standard activation functions for neural networks -- such as the Rectified Linear Unit (ReLU) -- are not polynomial. Nonetheless, networks with polynomial activations have been considered in the literature. Such networks are universal interpolators and approximators \parencite{constantinescu2023interpolation}. Various flavors of polynomial activations have been discussed, ranging from monomial versions of ReLU \parencite{berradi2018symmetric, li2019powernet}, to rational functions \parencite{boulle2020rational, telgarsky2017neural}, to piece-wise polynomial functions \parencite{lopez2019piecewise, hou2017convnets}. Finally, quadratic activations have appeared in neuroscience for modelling biological neural networks \parencite{adelson1985spatiotemporal}, and have sometimes been considered in theoretical studies of deep learning phenomena \parencite{gromov2023grokking, marchetti2024harmonics}. 

\section{BACKGROUND}\label{sec:back}
In this section, we introduce the necessary background around convolutional neural networks, as well as the relevant tools from algebraic geometry. 

\subsection{Polynomial Convolutional Networks}\label{sec:polconvs}
We start by introducing convolutional neural networks. For simplicity of notation, we focus on one-dimensional convolutions -- all theory and results can be extended verbatim to the higher-dimensional case. Given integers $k, s, d' \in \mathbb{N}$ representing filter size, stride, and output dimension respectively, the convolution between a filter $w \in \mathbb{R}^k$ and an input vector $x \in \mathbb{R}^d$, with $d = s(d' -  1) + k$, is the vector $ w \conv{s} x \in \mathbb{R}^{d'}$ defined for $ 0 \leq i < d'$ as: 
\begin{equation}
\label{eq:conv_star}
 (w \conv{s} x)[i] =   \sum_{0 \leq j < k} w[j] \  x[si + j].
\end{equation}
The convolution in Equation \eqref{eq:conv_star} is linear in $x$ and is represented by a $d' \times d$  \emph{Toeplitz matrix}. For example,  given $k=3$, $s=2$, and $d'=3$, the corresponding matrix is 
\[\begin{pmatrix}
    w[0]&w[1]&w[2]&0&0&0&0\\
    0&0&w[0]&w[1]&w[2]&0&0\\
    0&0&0&0&w[0]&w[1]&w[2]
\end{pmatrix}.\]

Moreover, the composition of convolutions can be rephrased as polynomial multiplications. We associate to the filter the following homogeneous bivariate polynomial of degree $k-1$ in $(a^s,b^s)$: 
\begin{equation}
    \label{eq:poly_correspondence}
    \pi_s(w) = \sum_{0 \leq i < k}w[i] \ a^{s(k-i-1)}b^{si} .
\end{equation}
Then the following property holds: an iterated convolution $v \conv{s} (w \conv{t} x)$ with respective strides $s, t$ coincides with a convolution $q \conv{st} x$ whose associated polynomial satisfies $\pi_{1}(q) = \pi_{t}(v) \ \pi_{1}(w)$.  

Convolutions can be composed in order to construct deep convolutional networks. To this end, fix an integer $L \in \mathbb{N}$ and sequences $\mathbf{k}, \mathbf{s} \in \mathbb{N}^{L}$, $\mathbf{d} \in \mathbb{N}^{L+1}$ such that $d_i = s_i(d_{i+1} + k_i)$ for all $i$. Moreover, consider a map $\sigma \colon \mathbb{R} \rightarrow \mathbb{R}$ playing the role of the activation function. 
\begin{defn}
A \emph{Convolutional Neural Network} (CNN) with weights $\mathbf{w} = (w_{0}, \ldots, w_{L-1}) \in \bigoplus_{0 \leq i < L}\mathbb{R}^{k_i}$ is the map $\varphi_\mathbf{w} \colon \mathbb{R}^{d_0} \rightarrow \mathbb{R}^{d_L}$ given by:
\begin{equation}
\varphi_\mathbf{w}(x) = w_{L-1} \conv{s_{L-1}} \sigma \left(\cdots \conv{s_1}  \sigma ( w_0 \conv{s_{0}} x   ) \right),
\end{equation}
where $\sigma$ is applied coordinate-wise.
\end{defn}
Given $r \in \mathbb{N}$, denote by $\sigma_r(x) = x^r$ the power function. CNNs with activation function $\sigma = \sigma_r$ are called \emph{polynomial}, and they are called \emph{linear} if $r = 1$ i.e., if $\sigma$ is the identity function. Polynomial CNNs define homogeneous polynomial functions of degree $r^{L-1}$, meaning that $\varphi_\mathbf{w} \in \symm{r^{L-1}}{d_0}^{d_L}$. Here, $\symm{\alpha}{\beta}$ denotes the space of symmetric tensors of degree $\alpha$ in $\beta$ variables, i.e., the vector space of dimension $\binom{\beta + \alpha - 1}{\alpha}$ whose canonical basis is given by monomials of degree $\alpha$ in $x[0], \ldots, x[\beta -1]$. 
\begin{defn}
The \emph{neuromanifold} of a polynomial CNN is the image in $\symm{r^{L-1}}{d_0}^{d_L}$ of the parametrization, i.e.,
\begin{equation}
\mathcal{M}_{\mathbf{d}, \mathbf{k}, \mathbf{s}, r} = \left\{\varphi_\mathbf{w} \ | \  \mathbf{w} \in \mathbb{R}^{|\mathbf{k}|} \right\}, 
\end{equation}
\end{defn}
where $|\mathbf{k}| = k_0 + \cdots + k_{L-1}$. By the Tarski-Seidenberg theorem, the neuromanifold is a semi-algebraic set, i.e., it can be defined by polynomial equalities and inequalities. However, we will show that it is actually an algebraic \emph{variety}, meaning that it is closed in Zariski topology of $\symm{r^{L-1}}{d_0}^{d_L}$, or, equivalently, that it can be defined by polynomial equalities alone -- see Proposition \ref{prop:zarboundary}. Lastly, throughout this work we will often consider complex coefficients. To this end, we denote by $\mathcal{M}_{\mathbf{d}, \mathbf{k}, \mathbf{s}, r}^\mathbb{C}$ the complexification of the neuromanifold. The latter can be defined by replacing real numbers $\mathbb{R}$ with complex ones $\mathbb{C}$ in all  definitions of this section. 

\subsection{Segre--Veronese Varieties}
Here, we recall the notions of projective spaces and Segre--Veronese varieties, together with their basic properties. As we will show, these are closely related to the neuromanifolds of CNNs. To this end, fix a field $\mathbb{K}$ (e.g., $\mathbb{R}$ or $\mathbb{C}$). For a vector space $V$ over $\mathbb{K}$, its \emph{projectification} $\mathbb{P} V$ consists of all the non-vanishing vectors up to rescaling or, equivalently, of all the lines
in $V$. Formally, it can be defined as the quotient $\mathbb{P} V =(V \setminus \{ 0\}) / (\mathbb{K} \setminus \{ 0\})$, where $\mathbb{K}$ acts on $V$ via scalar multiplication. Projective spaces will be relevant for us since, as we shall see, it is convenient to consider the parameters of polynomial CNNs -- i.e., the filters -- up to rescaling. 

Now, fix $k \in \mathbb{N}$, let $\mathbf{m}, \mathbf{p} \in \mathbb{N}^k$, and consider $k$ vector spaces $V_1, \ldots, V_k$ over $\mathbb{K}$, with $\dim(V_i) = p_i + 1$. 

\begin{defn}\label{def:segveremb}
    The \emph{Segre--Veronese embedding} is the map
    \begin{equation}
        \nu_{\mathbf{m}, \mathbf{p}} \colon \prod_{1 \leq i \leq k}\mathbb{P}V_i \rightarrow \mathbb{P} \bigotimes_{1 \leq i \leq k} 
      \symmm{m_i}{V_i}  
    \end{equation}
defined by taking tensor products of symmetric powers of vectors in the corresponding spaces. The \emph{Segre--Veronese variety} $\mathcal{V}_{\mathbf{m}, \mathbf{p}}$ is the image of $ \nu_{\mathbf{m}, \mathbf{p}}$.
\end{defn}
Explicitly, if $V_i = \mathbb{R}^{p_i + 1}$, then the Segre--Veronese variety consists of (tensor) products of $k$ monomials of corresponding degree $m_i$. In particular, $ \nu_{\mathbf{m}, \mathbf{p}} $ is an embedding, as suggested by the nomenclature, and $\mathcal{V}_{\mathbf{m}, \mathbf{p}}$ is a smooth projective variety of dimension $| \mathbf{p}| = p_1 + \cdots +  p_{k}$. When $k=1$, we refer to $\nu_{m,p}$ simply as Veronese embedding. For example, for $k=2$ and $p_1 = p_2 = m_1 = m_2 = 1$, $\mathcal{V}_{\mathbf{m}, \mathbf{p}}$ coincides with a smooth quadric in the three-dimensional projective space $\mathbb{P}(V_1 \otimes V_2)$, which can be represented as a hyperboloid -- see Figure \ref{fig:segre}. As we shall see, since $\varphi_\mathbf{w}(x)$ depends polynomially on $\mathbf{w}$, the prametrization and the neuromanifold of polynomial CNNs are closely related to the Segre-Veronese embedding and variety, respectively. 

\subsection{Euclidean Distance Degree}\label{sec:eucdd}
We now recall the \emph{Euclidean distance degree} \parencite{draisma2016euclidean} -- an invariant in algebraic geometry that will be central in our work. Intuitively, this invariant counts the number of critical points of the Euclidean distance function from (the smooth locus of) an algebraic variety to a fixed external point. Formally, let $X \subset \mathbb{R}^n$ be an algebraic variety, whose smooth locus is denoted by $X_{\textnormal{reg}}$. Let $A$ be a $n \times n$ symmetric positive-definite matrix inducing a distance $\textnormal{d}_A(x, y)^2 = (x-y)^\top A \ (x-y) $ for $x, y \in \mathbb{R}^n$. We consider the (squared) distance function from an anchor $u \in \mathbb{R}^n \setminus X$ to the smooth locus of $X$: 
\begin{equation}\label{eq:distfunc}
\begin{array}{ccc}
  X_{\textnormal{reg}}  & \rightarrow & \mathbb{R}_{\geq 0} \\
  x & \mapsto & \textnormal{d}_A(x, u)^2.
\end{array}
\end{equation}
For what follows, it is necessary to consider complex coefficients. To this end, we complexify the above map, extending it as $X^\mathbb{C}_\textnormal{reg} \rightarrow \mathbb{C}$, where $X^\mathbb{C} \subset \mathbb{C}^n$ denotes the complexification of $X$. 
\begin{defn}\label{def:edd}
The \emph{Euclidean distance degree} of $X$ with respect to $A$ is the  number of complex critical points of the complexified map $X^\mathbb{C}_\textnormal{reg} \rightarrow \mathbb{C}$ for generic~$u$. 
\end{defn}
Geometrically, a point $x \in X_\textnormal{reg}$ is critical for the distance map if, and only if, $x - u$ is perpendicular to the tangent space of $X$ at $x$ according to the scalar product induced by $A$. Moreover, it is possible to extend the notion of Euclidean distance degree to projective varieties $X \subset \mathbb{P}\mathbb{R}^{n}$ by considering their de-projectification, i.e., the corresponding affine cone in $\mathbb{R}^n$.

\begin{figure}[!ht]
    \centering
    \includegraphics[width=0.43\linewidth]{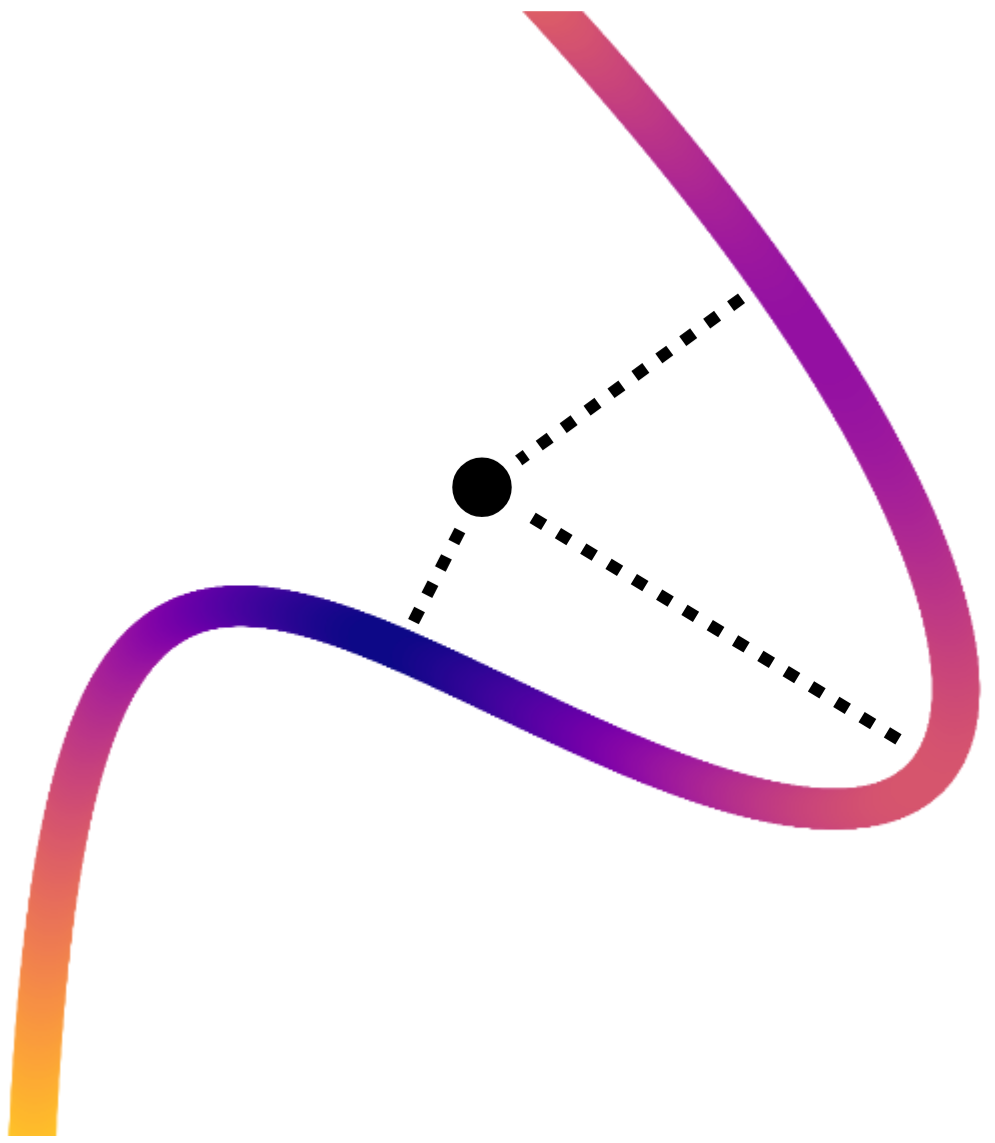}
    \caption{Distance function from an anchor to a curve, visualized as a color gradient. The critical values are denoted by dotted lines.}
    \label{fig:eddbezier}
\end{figure}

For generic $A$, the Euclidean distance degree is the same, achieving its maximal value \parencite[Theorem 6.11]{draisma2016euclidean}. This number is referred to as \emph{generic} Euclidean distance degree of $X$, denoted by $\edd(X) \in \mathbb{N}$. The following result -- which will be necessary in our forthcoming discussion -- provides a closed formula for the generic Euclidean distance degree of Segre--Veronese varieties. 

\begin{thm}[\cite{kozhasov2023minimal}]\label{th:segverdeg}
The generic Euclidean Distance degree of the Segre--Veronese variety is:
\begin{align*}
    \edd(\mathcal{V}_{\mathbf{m}, \mathbf{p}}) = &\sum_{0 \leq i \leq |\mathbf{p}|} (-1)^i \left( 2^{|\mathbf{p}| + 1 - i} - 1 \right)(|\mathbf{p}| - i)! \\ 
    &\sum_{\substack{|\mathbf{\alpha}|=i \\ \forall j \ \alpha_j \leq p_j}} \prod_{1\leq j \leq k} \frac{\binom{p_j + 1}{\alpha_j}}{(p_j - \alpha_j)!} \ m_j^{p_j - \alpha_j},
\end{align*}
where $|\mathbf{p}| = p_1 + \cdots + p_{k}$. 
\end{thm}
As anticipated in Section \ref{sec:intro}, machine learning models are trained to minimize, ideally, the distance to the ground-truth function. Therefore, the distance function over neuromanifolds and its critical points are closely related to the learning process. This motivates the study of the Euclidean distance degree of neuromanifolds, which we will explore for polynomial CNNs in Section \ref{sec:optim}.  

\section{CONVOLUTIONAL NEUROMANIFOLDS}\label{sec:convneu}
In this section, we study the neuromanifolds of polynomial CNNs, and present the main results of this work. We first discuss geometric properties -- such as dimension, projectification, and relation to Segre--Veronese varieties -- and then proceed to compute the generic Euclidean distance degree of the neuromanifold. We provide most of the proofs in the appendix -- see Section \ref{sec:remproof}. Moreover, in Section \ref{sec:toyexmp}, we provide a toy example where all the quantities discussed here can be computed explicitly.
\begin{figure*}[th!]
\captionsetup[subfigure]{justification=centering}
    \centering
    \hspace{1em}
    \begin{subfigure}[b]{.27\linewidth}
        \centering
        \includegraphics[width=\linewidth]{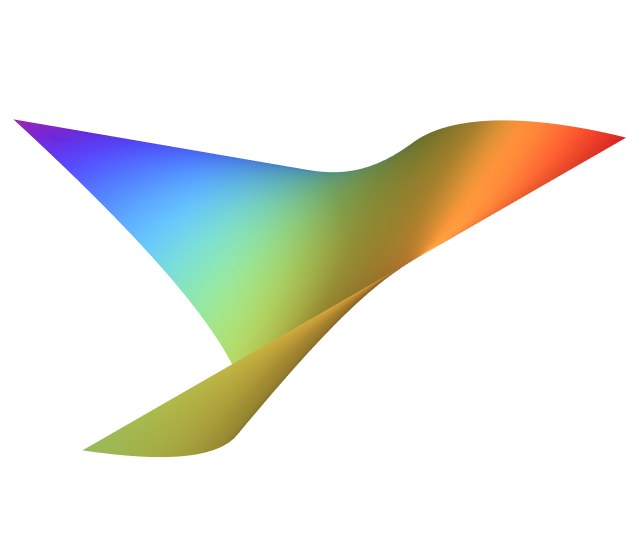}
        \subcaption*{$r = 2$}
    \end{subfigure}
    \hfill
    \begin{subfigure}[b]{.27\linewidth}
        \centering
        \includegraphics[width=\linewidth]{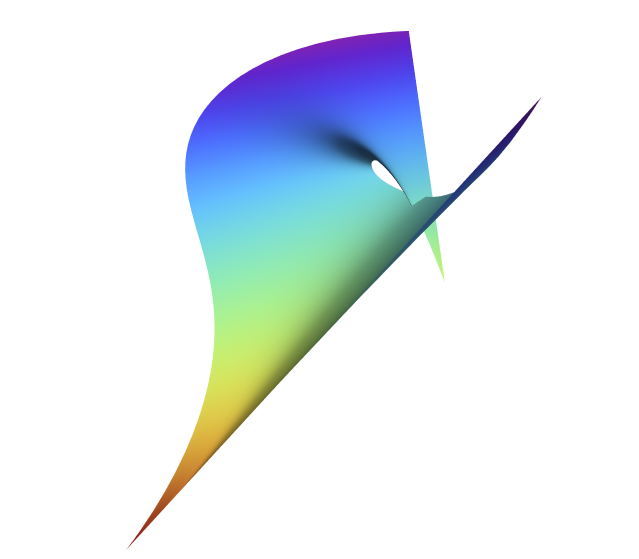}
        \subcaption*{$r = 5$}
    \end{subfigure}
    \hfill
    \begin{subfigure}[b]{.27\linewidth}
        \centering
        \includegraphics[width=\linewidth]{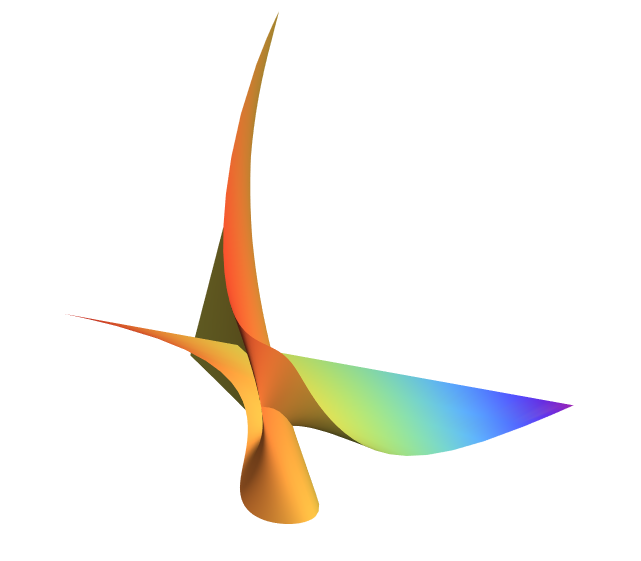}
        \subcaption*{$r = 10$}
    \end{subfigure}
    \caption{Visualization of two-dimensional charts of neuromanifolds (over $\mathbb{R}$) corresponding to $(k_0,k_1)=(2,2)$ projected orthogonally to $\mathbb{R}^3$, with varying activation degree. }
    \label{fig:neur}
\end{figure*}
\subsection{Geometry}\label{sec:geome}
Throughout this section, we will leverage on principles and tools from algebraic geometry. Therefore, it will be convenient to work with complex coefficients. For simplicity, we will abuse the notation and use the symbols from Section \ref{sec:polconvs} to denote the analogous complex object, obtained by replacing $\mathbb{R}$ with $\mathbb{C}$ in the corresponding definition. While all the results of this section are stated and proved over complex numbers, almost all of them extend a posteriori to the real case -- see Remark \ref{rem:realcase}.

We start with a simple but powerful property of convolutions.
\begin{lem}\label{lem:fullrk}
    For $w \in \mathbb{C}^k \setminus \{ 0 \}$ and $s \in \mathbb{N}$, the convolution map $x \mapsto w \conv{s} x$ has full rank. 
\end{lem}
\begin{proof}
See Section \ref{app:fullrk} in the appendix.  
\end{proof}
Consider now a polynomial CNN  $\varphi_\mathbf{w}$ parametrized by weights $\mathbf{w} \in \mathbb{C}^{|\mathbf{k}|}$. 


\begin{cor}\label{cor:basepts}
    If $\varphi_\mathbf{w}(x) = 0$ for all $x \in \mathbb{C}^{d_0}$, then $w_i = 0$ for some $i$. 
\end{cor}
\begin{proof}
Suppose that $w_i \not = 0 $ for all $i$. By Lemma \ref{lem:fullrk}, for all $i$ the map $x \mapsto \sigma_r( w_i \conv{s_i} x)$ is open, i.e., it sends open sets to open ones. Openness is preserved by composition and, in particular,
$\varphi_\mathbf{w} \not = 0$, as desired. 
\end{proof}

\begin{rmk}\label{rmk:basepts}
The property established by Corollary \ref{cor:basepts} is characteristic of (polynomial) CNNs, and does not hold for other network architectures. For example, a fully-connected network with arbitrary activations and at least two layers will not satisfy it. To illustrate this, consider a $2$-layer fully-connected network $W_1 \cdot \sigma(W_0 \cdot x)$, where $x \in \mathbb{R}^{d_0}$ is the input, $W_0 \in \mathbb{R}^{d_1 \times d_0}$ and $W_1 \in \mathbb{R}^{d_2 \times d_1}$ are the weight matrices, and $\sigma \colon \mathbb{R} \rightarrow \mathbb{R}$ is the activation function. Assuming $d_1 > 1$, if the first two rows of $W_0$ coincide and the other ones vanish, while the first two columns of $W_1$ are opposite, then the network vanishes on all $x \in \mathbb{R}^{d_0}$. 
\end{rmk}

As a consequence of Corollary \ref{cor:basepts}, it is possible to projectify the neuromanifold and its parametrization. Indeed, the result implies that the map $\mathbf{w} \mapsto \varphi_\mathbf{w}$ parametrizing the neuromanifold induces an algebraic morphism between the corresponding projective spaces: 
\begin{equation}\label{eq:projsegver}
    \overline{\varphi}:\prod_{0 \leq i < L} \mathbb{P} \mathbb{C}^{k_i} \rightarrow  
    \mathbb{P} \symmm{r^{L-1}}{\mathbb{C}^{d_0}}^{d_L}.  
\end{equation} 
We denote by $\mathbb{P}\mathcal{M}_{\mathbf{d}, \mathbf{k}, \mathbf{s}, r}^\mathbb{C}$ the image of $\overline{\varphi}$ and deem it \emph{projective neuromanifold}. We deduce the following relevant property of the (projective) neuromanifold. 
\begin{prop}\label{prop:zarclosed}
Both $\mathbb{P}\mathcal{M}_{\mathbf{d}, \mathbf{k}, \mathbf{s}, r}^\mathbb{C}$ and $\mathcal{M}_{\mathbf{d}, \mathbf{k}, \mathbf{s}, r}^\mathbb{C}$ are closed in the Zariski topology of their respective ambient spaces. 
\end{prop}
\begin{proof}
See Section \ref{app:zarclosed} in the appendix. 
\end{proof}
In order to analyse $\overline{\varphi}$, we provide another technical result that enables to rephrase polynomial CNN layers in terms of convolutions. To this end, fix a filter $w \in \mathbb{C}^k$, a stride $s$, and an exponent $r$. 
\begin{prop}\label{prop:converonese}
 For every $x \in \mathbb{C}^n$, $\sigma_r(w \conv{s} x)$ can be written as a convolution $\tilde{w} \conv{\tilde{s}} \tilde{x}$ with stride $\tilde{s} = s \binom{r+k-1}{r}$ and
filter size  $\tilde{k}=\binom{r+k-1}{r}$. Moreover, $\tilde{x}$ and $\tilde{w}$ consist of monomials in $x$ and $w$, respectively.  
\end{prop}
\begin{proof}   
See Section \ref{app:converonese} in the appendix. 
\end{proof}
Proposition \ref{prop:converonese} implies, inductively, that $\overline{\varphi}$ factors via the Segre--Veronese embedding (Definition \ref{def:segveremb}). Specifically, let $\mathbf{m} = (r^{L-1}, \ldots, r, 1)$ and $\mathbf{p} = (k_0 - 1, \ldots, k_{L-1} - 1)$. Then the following diagram commutes: 
\begin{equation*}
\begin{tikzcd}
	{\displaystyle \prod_{0 \leq i < L} \mathbb{P}\mathbb{C}^{k_i} } && \mathbb{P}\mathcal{M}_{\mathbf{d}, \mathbf{k}, \mathbf{s}, r}^\mathbb{C} \\
	& \mathcal{V}_{\mathbf{m}, \mathbf{p}}
	\arrow["\overline{\varphi}",from=1-1, to=1-3]
	\arrow["{ \nu_{\mathbf{m}, \mathbf{p}}}"', from=1-1, to=2-2]
	\arrow["\Lambda"', from=2-2, to=1-3]
\end{tikzcd}
\end{equation*}
Here, $\Lambda$ is (the projectification of) a linear map restricted to the Segre--Veronese variety. In the next section, this factorization  enables us to compute the neuromanifold's generic Euclidean distance degree.

We now provide the main results of this section. Recall that the differential of a differentiable map is defined as its local linearization, and is explicitly computable in given coordinates as the Jacobian. A differentiable map is \emph{regular} if its differential has maximal  rank at every point of the domain. 
\begin{thm}\label{thm:reg}
Assume that $r>1$. Then the projectified parametrization $\overline{\varphi}$ is regular. 
\end{thm}
\begin{proof}
The argument requires several technical computations around the differential of $\varphi$, whose details we provide in the appendix, Section \ref{sec:proofdet}. By Euler's theorem on homogeneous functions, showing that the differential of $\overline{\varphi}$ is injective at some $\mathbf{w}$ is equivalent to showing that the kernel of the differential of $\varphi$ at $\mathbf{w}$ has dimension $L-1$. The latter follows immediately from Proposition \ref{prop:kerdiff}, which provides $L-1$ independent generators for such kernel.     
\end{proof}
Next, we show that the map $\overline{\varphi}$ is \emph{birational}, meaning that it is generically one-to-one. In other words, `almost all' the fibers of $\overline{\varphi}$ are singletons. We also show that the remaining non-singleton fibers are finite.
\begin{thm}\label{thm:mainbir} 
Assume that $r > 1$. If $\overline{\varphi}_{\mathbf{w}} = \overline{\varphi}_{\mathbf{v}}$ for some $\mathbf{w}, \mathbf{v} \in  \prod_{i} \mathbb{P}\mathbb{C}^{k_i}  $, then $\mathbf{w}$ and $\mathbf{v}$ are related by a shift of their indices. More precisely, there exist $t_0, \ldots, t_{L-1} $ such that for all $i$:
\begin{align}
w_i &= (0, \ldots, 0, v_i[0], \ldots, v_i[k_i - t_i - 1]), \\ 
v_i &= (v_i[0], \ldots, v_i[k_i - t_i - 1], 0, \ldots, 0),
\end{align}
or vice versa. In particular,  $\overline{\varphi}$ is birational on its image, and all the fibers are finite.   
\end{thm}
\begin{proof}
The argument involves a delicate induction over the number of layers -- see Section \ref{app:mainbir} in the appendix. 
\end{proof}
Theorem \ref{thm:mainbir} can be intuitively interpreted as stating that the projective neuromanifold is isomorphic to a Segre--Veronese variety `almost everywhere'. For $r = 1$ -- i.e., for linear CNNs -- the same result has been shown by \cite{kohn2023function} with the additional assumption on strides $s_i > 1$ for all $0 \leq i \leq L - 2$. Moreover, in the same work the authors show that the parametrization of linear CNNs has critical points, in contrast to the polynomial case. The argument is different from ours and is based on the interpretation of convolutions as polynomial multiplication. 

\begin{rmk}\label{rmk:singsmall}
From the proof of Theorem \ref{thm:mainbir}, it is evident that the shifts must satisfy arithmetic constraints involving the strides. Specifically, consider the shifts $t_i$ together with their sign, which is set positive if $w_i$ has zeros on the left, and negative otherwise. Define recursively $\tilde{t}_{-1} = 0$, $\tilde{t}_i = t_i + \tilde{t}_{i-1} / s_{i-1}$ for $i \geq 0$. Then $\tilde{t}_i$ must be an integer, and moreover $\tilde{t}_{L-1}$ must vanish. By the definition of (iterated) convolutions, this implies that if we truncate each filter by removing its zeros on the right or on the left, the resulting CNN coincides with $\varphi_{\mathbf{w}}$, albeit on a restricted domain with less input variables. Put simply, the CNNs whose (projectified) fibers are not singletons are the ones that can be defined by smaller architectures, i.e., special
fibers arise from `subnetworks'.    
\end{rmk}

An immediate consequence of Theorem \ref{thm:mainbir} is an expression for the two fundamental invariants of the neuromanifold: the dimension and the degree. Recall that for a given projective/affine variety, the latter counts, roughly speaking, the number of intersections with a generic linear subspace of dimension equal to the co-dimension of the variety. 
\begin{cor}\label{cor:finiteparam}
Assume that $r>1$. The dimension and degree of the neuromanifold are: 
\begin{align}
    \textnormal{dim}(\mathcal{M}_{\mathbf{d}, \mathbf{k}, \mathbf{s}, r}^\mathbb{C}) &= |\mathbf{k}| - L + 1, \\
    \deg(\mathcal{M}_{\mathbf{d}, \mathbf{k}, \mathbf{s}, r}^\mathbb{C}) &= (|\mathbf{k}|-L)!\prod_{0 \leq j < L} \frac{r^{(L - j - 1)(k_j - 1)}}{(k_j - 1)!}.
\end{align}
\end{cor}
\begin{proof}
See Section \ref{app:finiteparam} in the appendix. 
\end{proof}
In particular, note that if all the filter sizes $k_i := k$ are equal, then the degree is $(L(k-1))! \ r^{L(L-1)(k-1) / 2 }/ (k-1)!^L$, which grows faster than $r^{L^2/2}$ w.r.t. $L$. Therefore, as anticipated in Section \ref{sec:summres}, the dimension grows linearly w.r.t. the depth of the network, while the degree grows (super-) exponentially.  

Another consequence of the main results in this section is a description of the singular points of the neuromanifold. Since $\overline{\varphi}$ is regular and finite, the singularities of $\mathbb{P}\mathcal{M}_{\mathbf{d}, \mathbf{k}, \mathbf{s}, r}^\mathbb{C}$ coincide with CNNs whose (projectified) fiber is not a singleton. Such singularities are of \emph{nodal} type, i.e., points where the neuromanifold self-intersects, creating a finite number of (potentially overlapping) tangent spaces -- see Figure \ref{fig:neur} (center and right) for an illustration. By Remark \ref{rmk:singsmall}, singularities arise from `subnetworks'. Since  $\mathcal{M}_{\mathbf{d}, \mathbf{k}, \mathbf{s}, r}^\mathbb{C}$ is the affine cone of its projectification, the same description of singularities holds for the non-projective neuromanifold, apart from the base of the cone $\varphi_\mathbf{w} = 0$, which is a non-nodal singularity. We summarize this in the following corollary. 
\begin{cor}\label{cor:singpts}
Assume that $r>1$. A point $\varphi_\mathbf{w} \in \mathcal{M}_{\mathbf{d}, \mathbf{k}, \mathbf{s}, r}^\mathbb{C}$ is singular if, and only if, $\varphi_\mathbf{w} = 0$ or the fiber of $\varphi_\mathbf{w}$ via $\overline{\varphi}$ contains multiple points. 
In particular, singular points correspond to CNNs that can be parametrized with a smaller architecture.
\end{cor}


\begin{rmk}\label{rem:realcase}
As anticipated, we remark that most of the results from this section hold over $\mathbb{R}$ as well. This is the case for Corollary \ref{cor:basepts} and Proposition \ref{prop:converonese} -- since they are purely algebraic -- as well as for Corollary \ref{cor:finiteparam} -- since the dimension is independent of the coefficients -- and for Theorem \ref{thm:reg} and Theorem \ref{thm:mainbir} -- since birationality and rank of differential are invariant with respect to restriction of coefficients. However, the proof of Proposition \ref{prop:zarclosed} does not hold over $\mathbb{R}$, since it leverages on properties requiring algebraic closedness. Yet, in the following, we show that the main results of this section imply that the neuromanifold is still closed in the Zariski topology over $\mathbb{R}$.  
\end{rmk}
\begin{prop}\label{prop:zarboundary}
Assume that $r>1$. Then both $\mathbb{P}\mathcal{M}_{\mathbf{d}, \mathbf{k}, \mathbf{s}, r}$ and $\mathcal{M}_{\mathbf{d}, \mathbf{k}, \mathbf{s}, r}$ are closed in the (real) Zariski topology of their respective ambient spaces. 
\end{prop}
\begin{proof}
See Section \ref{app:zarboundary} in the appendix.
\end{proof}
In particular, the (projective) neuromanifold is a real algebraic variety. This does not hold for linear CNNs ($r=1$), in which case the closure in the Zariski topology adds points to the neuromanifold \parencite{kohn2023function}. 
\subsection{Optimization}\label{sec:optim}
We now discuss aspects of optimization of a polynomial CNN for a regression task. In particular, by leveraging on the theory of the Euclidean distance degree introduced in Section \ref{sec:eucdd}, we compute an upper bound for the number of (complex) critical points over the neuromanifold for the regression objective. 

We start by introducing the regression objective. Consider a dataset, i.e., a finite subset $\mathcal{D} \subset \mathbb{R}^{d_0} \times \mathbb{R}^{d_L}$ representing input-output pairs. The square-error loss $\mathcal{L}_{\mathcal{D}}\colon \mathcal{M}_{\mathbf{d}, \mathbf{k}, \mathbf{s}, r} \rightarrow \mathbb{R}_{\geq 0}$ is given by: 
\begin{equation}
    \label{eq:loss}
    \mathcal{L}_\mathcal{D}(\varphi_\mathbf{w}) = \sum_{(x,y) \in \mathcal{D}} \| \varphi_{\mathbf{w}}(x) - y \|^2.
\end{equation}
Training a CNN on the dataset $\mathcal{D}$ amounts to minimizing $\mathcal{L}_\mathcal{D}$ over $\mathcal{M}_{\mathbf{d}, \mathbf{k}, \mathbf{s}, r}$. Intuitively, $\varphi_{\mathbf{w}}$ is optimized to interpolate $\mathcal{D}$, which is reminiscent of a closest-point problem over the neuromanifold, but for a discrete set $\mathcal{D}$ instead of an anchor function. Indeed, let $X$ and $Y$ be the matrices whose columns are, respectively, $\nu_{r^{L-1}, d_0 -1}(x) \in \symm{r^{L-1}}{d_0}$ and $y \in \mathbb{R}^{d_L}$ for $(x,y) \in \mathcal{D}$ (in some fixed order), where $\nu$ is the Veronese embedding. It follows that $X$ has full rank for a sufficiently large $|\mathcal{D}|$. Consider the positive definite symmetric matrix $A_\mathcal{D} = XX^\top \otimes I_{d_L}$. By \cite[Section 6]{kohn2022geometry}, we have
\begin{equation}
\label{eq:loss_dist}
 \mathcal{L}_\mathcal{D} (\varphi_\mathbf{w}) =  \textnormal{d}_{A_\mathcal{D}}(\varphi_\mathbf{w}, v_\mathcal{D})^2 + \textnormal{const},
\end{equation}
where $v_\mathcal{D} = YX^\top A_\mathcal{D}^{-1}$. The latter is generic for generic $\mathcal{D}$, since $Y$ can vary independently of $X$.

Motivated by this, we compute the generic Euclidean distance degree of the neuromanifold. 
\begin{prop}\label{prop:eddneurovar}
    Assume that $r > 1$ and let $\overline{k} = \textnormal{dim}(\mathcal{M}_{\mathbf{d}, \mathbf{k}, \mathbf{s}, r} )-1 = |\mathbf{k}| - L$. The generic Euclidean distance degree of the neuromanifold is given by: 
    \begin{align*}
         \edd\left(\mathcal{M}_{\mathbf{d}, \mathbf{k}, \mathbf{s}, r} \right) = \sum_{0 \leq i \leq\overline{k} } (-1)^i \left( 2^{\overline{k}  + 1 - i} - 1 \right)(\overline{k}  - i)!  \\ \sum_{\substack{|\mathbf{\alpha}|=i \\ \forall j \  \alpha_j < k_j}}  
         \prod_{0 \leq j < L} \frac{\binom{k_j}{\alpha_j}}{(k_j - \alpha_j - 1)!} \ r^{(L - j -1)(k_j - \alpha_j - 1)},
    \end{align*}
\end{prop}
\begin{proof}
See Section \ref{app:eddneurovar} in the appendix. 
\end{proof}
The above formula is challenging to analyze; numerical values are provided in Table \ref{tab:your_label}. Since it is known that for a projective variety $X$, $\textnormal{gED}(X) \geq \textnormal{deg}(X)$ \parencite[Theorem 5.4]{draisma2016euclidean}, a lower bound follows from  Corollary \ref{cor:finiteparam}. In particular, the growth is super-exponential w.r.t. the network's depth $L$. From the discussion in Section \ref{sec:eucdd},  
it follows that our formula upper-bounds the number of critical points over $(\mathcal{M}_{\mathbf{d}, \mathbf{k}, \mathbf{s}, r}^\mathbb{C})_\textnormal{reg}$ of the (complexified) loss function $\mathcal{L}_ \mathcal{D}$ for large generic datasets $\mathcal{D}$. Since real critical points over a variety remain such after complexification, the upper bound also holds for the critical points of the loss function over $(\mathcal{M}_{\mathbf{d}, \mathbf{k}, \mathbf{s}, r})_{\textnormal{reg}}$. 

A subtle point is that, in practice, optimization is performed via gradient descent over the parameter space, and not over the neuromanifold. The gradient flow of the loss -- and in particular the number of critical points -- could in principle differ due to the parametrization. However, this is not the case for polynomial CNNs: since the projectified parametrization is regular (Theorem \ref{thm:reg}), $\mathbf{w}$ is critical for $\mathcal{L}_{\mathcal{D}} \circ \varphi$ if, and only if, $\varphi_\mathbf{w}$ is critical for $\mathcal{L}_{\mathcal{D}}$. Therefore, our formula equivalently upper-bounds the number of critical points of $\mathcal{L}_\mathcal{D} \circ \varphi$ over $\varphi^{-1}((\mathcal{M}_{\mathbf{d}, \mathbf{k}, \mathbf{s}, r}^\mathbb{C})_{\textnormal{reg}}) \subseteq \mathbb{C}^{|\mathbf{k}|}$, up to the fibers of the parametrization, i.e., up to rescaling each filter. 

Lastly, note that so far we have removed the singular points from the neuromanifold in the above constructions. Therefore, we now discuss the role of singular points of the neuromanifold in the optimization. The following result shows that, excluding the trivial vanishing CNN, the (parameters of) such singular points are not critical for a generic distance minimization problem. 

\begin{prop}\label{prop:nosing}
Assume that $r>1$. Let $V \subseteq \symm{r^{L-1}}{d_0}^{d_L}$ be a linear subspace containing the neuromanifold, and $A$ be a positive-definite quadratic form on $V$. For a generic element $u \in V$, if $\mathbf{w}  \in \mathbb{R}^{|\mathbf{k}|}$ is critical for $\textnormal{d}_A(\varphi_{\bullet}, u)^2$, then either $\varphi_\mathbf{w} = 0$ or $\varphi_\mathbf{w} \in (\mathcal{M}_{\mathbf{d}, \mathbf{k}, \mathbf{s}, r})_\textnormal{reg}$.
\end{prop}
\begin{proof}
See Section \ref{app:nosing} in the appendix. 
\end{proof}
Therefore, if we exclude vanishing filters, we obtain an upper bound on the number of critical points in parameter space, up to rescaling each filter, of $\mathcal{L}_{\mathcal{D}} \circ \varphi$ for large generic $\mathcal{D}$. 

\section{CONCLUSIONS AND FUTURE WORK}
In this work, we have studied aspects of the geometry and optimization of polynomial CNNs. In particular, we have proven that the projectified parametrization is regular and finite birational, derived the dimension of the neuromanifold, and related the latter to the Segre--Veronese variety. Moreover, we have rephrased the optimization of the regression loss as a distance minimization problem, and leveraged on the theory of the Euclidean distance degree to upper-bound the number of (complex) critical points for a generic large dataset. 

Our tools involve general yet powerful results from algebraic geometry. Therefore, similar arguments might be applicable to other neural architectures. For example, \emph{graph neural networks} \parencite{kipf2016semi, bronstein2021geometric} are nowadays popular in a variety of domains, and are closely related to CNNs. Therefore, exploring applications of algebraic geometry to the neuromanifold of such networks represents an interesting line for future investigation. 

A limitation of the Euclidean distance degree is that it quantifies the number of critical points, without distinguishing local maxima, minima, or saddle points. Since local minima are the stable equilibria of gradient descent to which the latter converges, counting them represents a fundamental challenge from the perspective of optimization and learning. To this end, tools from \emph{Morse theory} \parencite{milnor1963morse} -- an approach relating critical points with prescribed index to topological invariants -- might be applicable to neuromanifolds, providing insights into local minima of the loss function. This constitutes another interesting direction to investigate.

From a broader perspective, a fundamental challenge lies in extending the study of neuromanifolds beyond the algebraic setting, i.e., for activation functions that are not polynomial (e.g., sigmoid and ReLU). To this end, a promising approach is via polynomial approximation. By the Weierstrass Approximation Theorem, any continuous function can be (locally) approximated by polynomials with arbitrary precision. Therefore, tools from algebraic geometry can be potentially applied beyond the algebraic realm, at least approximately. For CNNs, our work is the first to tackle the polynomial case; extending the results to non-polynomial CNNs (e.g. ReLU CNNs) via approximation is a natural and interesting next step. 

\section*{Acknowledgements}
We thank Rainer Sinn for helpful discussions on real algebraic geometry and the recovery of polynomials from partial terms.
This work was partially supported by the Wallenberg AI, Autonomous Systems and Software Program (WASP) funded by the Knut and Alice Wallenberg Foundation.

\printbibliography

\newpage 
\appendix 

\onecolumn 
\section{ON THE DIFFERENTIAL OF THE PARAMETRIZATION}\label{sec:proofdet}
In this section, we prove technical results on the differential of the parametrization of polynomial CNNs. The regularity of the projectified parametrization follows immediately -- see Theorem \ref{thm:reg}. We adhere to the convention from Section \ref{sec:convneu} and work with complex scalars. As discussed in Remark \ref{rem:realcase}, the results hold, a posteriori, over $\mathbb{R}$ as well. Moreover, we denote by $\textnormal{J}_\mathbf{w} \varphi$ the differential of $\varphi$ at $\mathbf{w}$. Since both the domain and co-domain of $\varphi$ are vector spaces, the differential can be seen as a linear map $\textnormal{J}_\mathbf{w} \varphi \colon \mathbb{C}^{|\mathbf{k}|} \rightarrow \symm{r^{L-1}}{d_0}^{d_L}$. Moreover, for $L>1$, we denote $\mathbf{w}' = (w_{0}, \ldots, w_{L-2})$. Then $\varphi_{\mathbf{w}} = w_{L-1} \conv{s_{L-1}} \sigma_r(\varphi_{\mathbf{w}'})$, and by applying the Leibniz derivation rule we see that $\textnormal{J}_\mathbf{w} \varphi$ sends a tangent vector $\dot{\mathbf{w}} = (\dot{w}_{0}, \ldots, \dot{w}_{L-1}) \in \mathbb{C}^{|\mathbf{k}|}$ to:
\begin{equation}\label{eq:diffhadam}
\dot{w}_{L-1} \conv{s_{L-1}} \sigma_r\left(  \varphi_{\mathbf{w}'} \right) + r w_{L-1} \conv{s_{L-1}} \sigma_{r-1}\left(  \varphi_{\mathbf{w}'} \right) \odot \left( \textnormal{J}_\mathbf{w'} \varphi \right)(\dot{\mathbf{w}}'),
\end{equation}
where $\odot$ denotes the Hadamard product, i.e., the point-wise product between polynomial functions. 

\begin{lem}\label{lemm:homogendiff}
Given $\lambda_0, \ldots, \lambda_{L-1} \in \mathbb{C}$, the differential $\textnormal{J}_\mathbf{w} \varphi$ sends $(\lambda_{0}w_{0}, \ldots, \lambda_{L-1} w_{L-1})$ to:
\begin{equation}
    (\lambda_{L-1} + r \lambda_{L-2} + \cdots + r^{L-1}\lambda_0) \ \varphi_\mathbf{w}. 
\end{equation}
\end{lem}
\begin{proof}
This follows immediately by substituting $\dot{w}_i = \lambda_i w_i$ for all $i$ in Equation \ref{eq:diffhadam} and by reasoning inductively on the number of layers $L$. 
\end{proof}

\begin{prop}\label{prop:kerdiff}
Suppose that $w_i \not = 0 $ for all $i$. Then the kernel of $\textnormal{J}_\mathbf{w} \varphi$ is generated, as a linear subspace of $\mathbb{C}^{|\mathbf{k}|}$, by $(0,\ldots,   0, w_{L-2}, -rw_{L-1})$, $(0, \ldots, 0, w_{L-3}, -rw_{L-2},0)$, $\ldots$, $(w_0, -rw_1, 0, \ldots, 0)$. 
\end{prop}
\begin{proof}
We proceed by induction over the number of layers $L$. For $L=1$,  $\varphi_\mathbf{w}(x) = w_0 \conv{s_0} x$, which is linear in $w_0$ and does not vanish for all $x$ since $w_0 \not = 0$. The differential is therefore injective, and its kernel is trivial. 

Suppose now that $L > 1$. In order to compute the kernel $\textnormal{J}_\mathbf{w} \varphi$, we need to solve
\begin{equation}\label{eq:kerndiff}
   \dot{w}_{L-1} \conv{s_{L-1}} \sigma_r\left(  \varphi_{\mathbf{w}'} \right) + r w_{L-1} \conv{s_{L-1}} \sigma_{r-1}\left(  \varphi_{\mathbf{w}'} \right) \odot \left( \textnormal{J}_\mathbf{w'} \varphi \right)(\dot{\mathbf{w}}') = 0. 
\end{equation}
If $\dot{\mathbf{w}}'$ belongs to the kernel of $\textnormal{J}_\mathbf{w'} \varphi$, then since $\varphi_{\mathbf{w}'} \not = 0$ by hypothesis (see Corollary \ref{cor:basepts}), we must have $\dot{w}_{L-1} = 0$. It follows from the inductive hypothesis that $\dot{w}$ is a linear combination of the tuples of filters in the statement. We therefore assume that $\left( \textnormal{J}_\mathbf{w'} \varphi\right)(\dot{\mathbf{w}}') \not = 0$. Moreover, without loss of generality, we assume that  $w_{L-1}[0] \not = 0$. 

We will now decompose $\varphi_\mathbf{w}$ as a sum of CNNs with smaller filter size on the last layer. To this end, consider the tuple of filters:
\begin{equation}
\mathbf{w}^+=(w_{0}, \ldots, w_{L-2}, (w_{L-1}[0])), \hspace{4em} \mathbf{w}^-=(w_{0}, \ldots, w_{L-2}, (0, w_{L-1}[1], \ldots, w_{L-1}[k_{L-1} -1])).
\end{equation}
Then $\varphi_\mathbf{w} = \varphi_{\mathbf{w}^+} + \varphi_{\mathbf{w}^-}$. Now, let $i$ be the minimal index such that $x[i]$ appears in $\varphi_{\mathbf{w}'}[0]$, seen as a scalar-valued polynomial in the input variable $x$. Recall the basic equivariance property of CNNs, meaning that shifting indices in the output variable is equivalent to shifting them in the input one (assuming the latter shift accounts for the strides). This implies that $x[i]$ does not appear in neither $\left( \textnormal{J}_\mathbf{w^-} \varphi \right)(\dot{\mathbf{w}}^-)$ nor $\varphi_{\mathbf{w}^-}$. Since $0 = \textnormal{J}_\mathbf{w} \varphi = \textnormal{J}_{\mathbf{w}^+} \varphi + \textnormal{J}_{\mathbf{w}^-} \varphi$, $x[i]$ does not appear in $\textnormal{J}_{\mathbf{w}^+} \varphi$ as well. Write: 
\begin{equation}
\varphi_{\mathbf{w}'}[0] = x[i]f + g, \hspace{7em} \left( \textnormal{J}_\mathbf{w'} \varphi \right)(\dot{\mathbf{w}}')[0] = x[i]a + b,
\end{equation}
where $f,g,a,b$ are polynomials such that $f \not =  0$ and $x[i]$ does not appear in $g$ nor $b$. The (vanishing) terms containing $x[i]$ in $\textnormal{J}_{\mathbf{w}^+} \varphi$ have the form:
\begin{equation}
    0 = \dot{w}_{L-1}[0] \left((x[i]f + g)^r - g^r \right) + rw_{L-1}[0]\left((x[i]f + g)^{r-1}(x[i]a + b) - g^{r-1}b\right).
\end{equation}
By manipulating the above expression, we obtain:
\begin{equation}
(x[i]f + g)^{r-1}\left(\dot{w}_{L-1}[0](x[i]f + g) + rw_{L-1}[0](x[i]a + b) \right) = g^{r-1}( \dot{w}_{L-1}[0]g + r w_{L-1}[0]b).
\end{equation}
The right-hand side of the above equation does not contain $x[i]$. Since $f \not = 0$, it follows that $\dot{w}_{L-1}[0](x[i]f + g) + rw_{L-1}[0](x[i]a + b) = 0$, implying:
\begin{equation}
\varphi_{\mathbf{w}'}[0] = x[i]a + b = \lambda \ (x[i]f + g) = \lambda \left( \textnormal{J}_\mathbf{w'} \varphi \right)(\dot{\mathbf{w}}')[0], \hspace{3em} \lambda := -\frac{\dot{w}_{L-1}[0]}{w_{L-1}[0]}.
\end{equation}
By looking at (the $0$-th entry of) Equation \ref{eq:kerndiff}, we deduce $\dot{w}_{L-1} = -r \lambda w_{L-1}$. By Lemma \ref{lemm:homogendiff}, $\dot{\mathbf{w}}' = (\lambda_{0}w_{0}, \ldots, \lambda_{L-2}w_{L-2})$ for some $\lambda_i \in \mathbb{C} \setminus \{0\}$ such that $\lambda_{L-2} + r \lambda_{r-3} + \cdots + r^{L-2} \lambda_0 = \lambda$. But then
\begin{align}
\dot{\mathbf{w}} &= (\lambda_0 w_0, \ldots,  \lambda_{L-2}w_{L-2}, -r \lambda w_{L-1}) = \\
&= \lambda (0, \ldots, 0, w_{L-2}, -rw_{L-1}) + (\lambda_{0} w_0, \ldots, \lambda_{L-3}w_{L-3}, (\lambda_{L-2} - \lambda)w_{L-2}, 0).
\end{align}
Since $(\lambda_{L-2} - \lambda) + r \lambda_{L-3} + \cdots + r^{L-2}\lambda_0 = 0$, it follows from Lemma \ref{lemm:homogendiff} that the second summand of the above expression belongs to the kernel of $ \textnormal{J}_\mathbf{w'} \varphi$, and the desired result follows from the inductive hypothesis.
\end{proof}

\section{ADDITIONAL PROOFS}\label{sec:remproof}
\subsection{Proof of Lemma \ref{lem:fullrk}}\label{app:fullrk}
\begin{proof}
Since the domain of a convolution has higher dimension than the co-domain, we need to show that the rows $A_0, \ldots, A_{d'-1}$ of the corresponding Toeplitz matrix are linearly independent. By induction on the filter size $k$, we can assume $w[0] \not = 0$. Suppose that $\sum_i a_i A_i = 0$ for some scalars $a_0, \ldots, a_{d'-1}$. Since the first column of the Toeplitz matrix has only one non-vanishing entry, we deduce $a_0 = 0$. But then, since the $s$-th column is $(w[s-1], w[0], 0, \ldots, 0)$, we similarly deduce that $a_1 = 0$, and so on. 
\end{proof}

\subsection{Proof of Proposition \ref{prop:zarclosed}}\label{app:zarclosed}
\begin{proof}
Projective morphisms over algebraically-closed fields are closed \parencite[II, \S 4, Theorem 4.9]{hartshorne2013algebraic} and, in particular, have closed image. Therefore, $\mathbb{P}\mathcal{M}_{\mathbf{d}, \mathbf{k}, \mathbf{s}, r}^\mathbb{C}$ is closed in $\mathbb{P} \symmm{r^{L-1}}{\mathbb{C}^{d_0}}^{d_L}$. Since $\varphi$ is homogeneous, the neuromanifold coincides with the affine cone of its projectification. Since affine cones of varieties are varieties, $\mathcal{M}_{\mathbf{d}, \mathbf{k}, \mathbf{s}, r}^\mathbb{C}$ is closed in $\symmm{r^{L-1}}{\mathbb{C}^{d_0}}^{d_L}$.     
\end{proof}

\subsection{Proof of Proposition \ref{prop:converonese}}\label{app:converonese}
\begin{proof} 
 Newton's multinomial expansion yields: 
  \begin{equation*}\label{eq:newtexp}
      \displaystyle (w \conv{s} x)[i]^r = \sum_{|\mathbf{a}| = r}\binom{r}{\mathbf{a}} \prod_{0 \leq j < k} w[j]^{a_j} \prod_{0 \leq j < k}x[is + j]^{a_j},
\end{equation*}
where $|\mathbf{a}| = a_0 + \cdots + a_{k-1}$ and $\binom{r}{\mathbf{a}} = \frac{r!}{a_0! \cdots a_{k-1}!}$. We construct $\widetilde{x}$ by indexing its coordinates via pairs $(i, \mathbf{a})$, where $1 \leq i < d - k$ and $\mathbf{a}$ is a multi-index with $|\mathbf{a}| = r$, in some fixed order. The corresponding coordinate of $\tilde{x}$ is $\prod_j x[i +j]^{a_j}$. Similarly, $\tilde{w}$ is indexed by $\mathbf{a}$, with corresponding coordinate $\binom{r}{\mathbf{a}}\prod_j w[j]^{a_j}$, i.e., $\tilde{w}$ is a rescaling of the Veronese embedding of $w$ of degree $r$. The claim follows immediately. 
\end{proof}

\subsection{Proof of Theorem \ref{thm:mainbir}}\label{app:mainbir}
We first provide a general result on homogeneous polynomials.
\begin{lem}\label{lemm:eulerproj}
Assume that $r > 1$. Let $d,n \in \mathbb{N}$, $0 \leq i < n$, and $p$ be a multivariate homogeneous polynomial of degree $d$ in $n$ variables over a field of characteristic $0$ (e.g., $\mathbb{R}$ or $\mathbb{C}$). Then the terms containing $x[i]$ in $p^r$ determine $p$ up to multiplicative scalar. 
\end{lem}
\begin{proof}
Write $p = \sum_{0 \leq j \leq d}x[i]^jq_j$, where $q_j$ is a homogeneous polynomial of degree $d-j$ not containing $x[i]$. Newton's multinomial expansion yields:
\begin{equation}
p^r = \sum_{|\mathbf{a}| = r}\binom{r}{\mathbf{a}} \prod_{0 \leq j \leq d} x[i]^{j a_j}q_j^{a_j} =  \sum_{|\mathbf{a}| = r}\binom{r}{\mathbf{a}} x[i]^{\sum_{j}j a_j} \prod_{0 \leq j \leq d} q_j^{a_j} .  
\end{equation}
By following $t := \sum_{0 \leq j \leq d}ja_j$ in decreasing order, we can inductively recover $q_j$ up to scalar in decreasing order w.r.t. $j$. More specifically, starting from $t = dr$, the coefficient of $x[i]^t$ is $q_d^{d}$, from which we recover $q_d$. Then, for $t = dr -1$, the coefficient of $x[i]^t$ is $q_d^{r-1}q_{d-1}$, from which we can now recover $q_{d-1}$, and so on.
\end{proof}
We are now ready to prove Theorem \ref{thm:mainbir}.
\begin{proof}
We will prove the following equivalent statement on the non-projectified parametrization: if $\varphi_\mathbf{w}[0] = \varphi_\mathbf{v}[k]$ for some $0 < k < d_L$ and some $\mathbf{w}, \mathbf{v} \in \mathbb{C}^{|\mathbf{k}|}$ such that $w_i, v_i \not = 0$ for all $0 \leq i < L$, then $\mathbf{w}$ and $\mathbf{v}$ are related, up to rescaling, by a shift of their indices. To this end, we proceed by induction on $L$. For $L = 1$, the desired result follows from the fact that the linear map defined by a convolution determines the corresponding filter. Suppose now that $L > 1$ and that $\varphi_\mathbf{w}[0] = \varphi_\mathbf{v}[k]$. The inductive argument is similar in spirit to the one from the proof of Proposition \ref{prop:kerdiff}. Denote  $\mathbf{w}' = (w_{0}, \ldots, w_{L-2})$, and analogously for $\mathbf{v}'$. Then for all $x$: 
\begin{equation}\label{eq:shiftfil}
\varphi_\mathbf{w}(x)[0] = \sum_{0 \leq j < k_{L-1}} w_{L-1}[j]\ (\varphi_{\mathbf{w}'}(x)[j])^r = \varphi_\mathbf{v}(x)[k] = \sum_{0 \leq j < k_{L-1}} v_{L-1}[j]\ (\varphi_{\mathbf{v}'}(x)[s_{L-1}k + j])^r. 
\end{equation}
Let $i$ be the minimal index such that $x[i]$ appears in the above expression, seen as a scalar-valued polynomial in the input variable $x$. Moreover, let $m$ and $n$ be the minimal indices such that $w_{L-1}[m] \not = 0$ and $v_{L-1}[n] \not = 0$, respectively. Note that among all the summands in Equation \ref{eq:shiftfil}, $x[i]$ appears only in $w_{L-1}[m] \ (\varphi_{\mathbf{w}'}(x)[m])^r $ and $v_{L-1}[n] \ (\varphi_{\mathbf{v}'}(x)[s_{L-1}k + n])^r$. Therefore, the terms involving $x[i]$ in these two summands must coincide. By Lemma \ref{lemm:eulerproj}, $\varphi_{\mathbf{w}'}(x)[m] = \lambda \ \varphi_{\mathbf{v}'}(x)[s_{L-1}k + n]$ for some $\lambda \in \mathbb{C} \setminus \{ 0\}$. From the equivariance properties of convolutions, it follows that $\varphi_{\mathbf{w}'}(x)[j] = \lambda \ \varphi_{\mathbf{v}'}(x)[s_{L-1}k + n - m + j]$ for all $0 \leq j < d_{L-2}$. From the inductive hypothesis, for all $0\leq t < L-1$, $w_t$ and $v_t$ are related, up to rescaling, by a shift of their indices.  Moreover, Equation \ref{eq:shiftfil} expands to:
\begin{equation}\label{eq:sumindxs}
\sum_{n \leq j < k_{L-1}} v_{L-1}[j] \ (\varphi_{\mathbf{v}'}(x)[s_{L-1}k + j])^r = \sum_{m \leq j < k_{L-1}} w_{L-1}[j] \lambda^r \  (\varphi_{\mathbf{v}'}(x)[s_{L-1}k + n - m + j])^r.
\end{equation}
Again, in the above expression $x[i]$ appears only in the summands $v_{L-1}[n]\ (\varphi_{\mathbf{v}'}(x)[s_{L-1}k + n])^r $ and $w_{L-1}[m] \lambda^r \  (\varphi_{\mathbf{v}'}(x)[s_{L-1}k + n])^r$. Since these summands must coincide, we conclude that $v_{L-1}[n] = w_{L-1}[m]\lambda^r$. But then the indices of the sums in Equation \ref{eq:sumindxs} start from $n+1$ and $m+1$. By iterating this argument, we conclude that $v_{L-1}[n - m + j] = w_{L-1}[j]\lambda^r$ for all $j$, i.e., the filters of the last layer coincide up to rescaling and shifting the indices, as desired. 

Lastly, note that the birationality statement on $\overline{\varphi}$ follows immediately. Indeed, the above argument shows that the fiber of $\overline{\varphi}$ at $\varphi_{\mathbf{w}}$ is not a singleton only if some filters in $\mathbf{w}$ are padded by zeros on the left or on the right, which is a negligible condition, i.e., it does not hold almost everywhere. The remaining fibers are characterized by shifting indices of filters, and are therefore finite. 
\end{proof}

\subsection{Proof of Corollary \ref{cor:finiteparam}}\label{app:finiteparam}
\begin{proof}
Birational maps preserve the dimension. Therefore, the dimension of the projective neuromanifold is:
\begin{equation}
\textnormal{dim}(\mathbb{P}\mathcal{M}_{\mathbf{d}, \mathbf{k}, \mathbf{s}, r}^\mathbb{C}) = \dim (\mathbb{P} \mathbb{C}^{k_0} \times \cdots \times \mathbb{P} \mathbb{C}^{k_{L-1}}) = \sum_{0 \leq i < L} (k_i - 1) = |\mathbf{k}| - L.
\end{equation}
  Since $\textnormal{dim}(\mathbb{P}\mathcal{M}_{\mathbf{d}, \mathbf{k}, \mathbf{s}, r}^\mathbb{C}) = \textnormal{dim}(\mathcal{M}_{\mathbf{d}, \mathbf{k}, \mathbf{s}, r}^\mathbb{C}) -1$,  the dimension of the neuromanifold is $|\mathbf{k}| - L + 1$.

A birational linear (projective) map from an algebraic variety whose kernel does not intersect the variety preserves the degree. Since the projective neuromanifold is the image of a Segre--Veronese variety via a birational linear map, we have that $\deg(\mathbb{P}\mathcal{M}_{\mathbf{d}, \mathbf{k}, \mathbf{s}, r}^\mathbb{C}) = \deg(\mathcal{V}_{\mathbf{m}, \mathbf{p}})$. The latter can be computed via  \cite[Proposition 6.11]{kozhasov2023minimal}, as follows:
\begin{equation}
    \label{eq:deg_neuro}
    \deg(\mathcal{V}_{\mathbf{m}, \mathbf{p}}) = (|\mathbf{k}|-L)!\prod_{0 \leq j < L} \frac{r^{(L - j - 1)(k_j - 1)}}{(k_j - 1)!}.
\end{equation}
Finally, since the degree of a projective variety coincides with the one of its affine cone, we have $\deg(\mathcal{M}_{\mathbf{d}, \mathbf{k}, \mathbf{s}, r}^\mathbb{C}) = \deg(\mathbb{P}\mathcal{M}_{\mathbf{d}, \mathbf{k}, \mathbf{s}, r}^\mathbb{C})$. 
\end{proof}

\subsection{Proof of Proposition \ref{prop:zarboundary}}\label{app:zarboundary}
\begin{proof}
This follows from the regularity of the parametrization. Specifically, for an image of a real algebraic map, the relative boundary of the image inside its Zariski closure is contained in the branch locus of the complexification of the map. Since the complexified projectified parametrization $\overline{\varphi}$ is regular (Theorem \ref{thm:reg}), it is unramified, i.e., the branching locus is empty. Therefore, $\mathbb{P}\mathcal{M}_{\mathbf{d}, \mathbf{k}, \mathbf{s}, r}$ is closed in the Zariski topology, and the same holds for its affine cone $\mathcal{M}_{\mathbf{d}, \mathbf{k}, \mathbf{s}, r}$.
 \end{proof}

\subsection{Proof of Proposition \ref{prop:eddneurovar}}\label{app:eddneurovar}
\begin{proof}
From Section \ref{sec:geome}, we know that $\overline{\varphi}$ factors into the Segre--Veronese embedding $\nu_{\mathbf{m}, \mathbf{p}}$, with $\mathbf{m} = (r^{L-1}, \ldots, r, 1)$ and $\mathbf{p} = (k_0 - 1, \ldots, k_{L-1} - 1)$, followed by a linear map. By \cite[Corollary 6.1]{draisma2016euclidean}, linear projections of varieties of co-dimension $\geq 2$ do not alter the generic Euclidean distance degree. Therefore, $\edd(\mathbb{P}\mathcal{M}_{\mathbf{d}, \mathbf{k}, \mathbf{s}, r}) = \edd(\mathcal{V}_{\mathbf{m}, \mathbf{p}})$. The generic Euclidean distance degree of the Segre--Veronese variety is given by Theorem \ref{th:segverdeg}, from which our formula follows via elementary algebraic manipulations. Note that, by definition, the Euclidean distance degree of a projective variety coincides with the one of its affine cone. Since $\varphi$ is homogeneous, the affine cone of $\mathbb{P}\mathcal{M}_{\mathbf{d}, \mathbf{k}, \mathbf{s}, r}$ is $\overline{\mathcal{M}}_{\mathbf{d}, \mathbf{k}, \mathbf{s}, r}$, concluding the proof.     
\end{proof}

Below, we provide a table with numerical values of the generic Euclidean distance degree. 
\begin{table}[!h]
    \centering
    \renewcommand{\arraystretch}{1.5}
    \setlength{\tabcolsep}{12pt} 
    \begin{tabular}{c|ccccc}
        \diagbox[width=6em, height=2em]{\footnotesize $r$}{\footnotesize $k$} & \footnotesize $2$ & \footnotesize $3$ & \footnotesize $4$ & \footnotesize $5$ & \footnotesize $6$ \\
        \hline
        \footnotesize $1$ & 6 & 39 & 284 & 2205 & 17730 \\
        \footnotesize $2$ & 14 & 219 & 3772 & 68405 & 1277898 \\
        \footnotesize $3$ & 22 & 543 & 14684 & 417005 & 12186066 \\
        \footnotesize $4$ & 30 & 1011 & 37244 & 1439205 & 57202074 \\
        \footnotesize $5$ & 38 & 1623 & 75676 & 3699005 & 185917794 \\
        \footnotesize $6$ & 46 & 2379 & 134204 & 7933205 & 482134890 \\
    \end{tabular}
    \caption{The generic Euclidean distance degree of neuromanifolds with $L=2$ and $k:=k_1=k_2$.}
    \label{tab:your_label}
\end{table}

\subsection{Proof of Proposition \ref{prop:nosing}}\label{app:nosing}
\begin{proof}
Note that $\mathbf{w}$ is critical for $\textnormal{d}_A(\varphi_{\bullet}, u)^2$ if, and only if, $\varphi_{\mathbf{w}} - u$ is perpendicular, according to the scalar product induced by $A$, to the image of the differential of $\varphi$ at $\mathbf{w}$. In other words, $u$ must belong to the relative normal bundle of $\varphi$. By Proposition \ref{cor:finiteparam} and Theorem \ref{thm:reg}, such bundle consists, at a given $\varphi_\mathbf{w}$, of a finite number of affine subspaces of co-dimension $|\mathbf{k}| - L + 1 = \textnormal{dim}((\mathcal{M}_{\mathbf{d}, \mathbf{k}, \mathbf{s}, r}))$. If we constrain $\varphi_\mathbf{w}$ to be singular, then the restricted bundle has co-dimension $\textnormal{dim}(\mathcal{M}_{\mathbf{d}, \mathbf{k}, \mathbf{s}, r}) - \textnormal{dim}((\mathcal{M}_{\mathbf{d}, \mathbf{k}, \mathbf{s}, r})_\textnormal{sing}) > 0$. Belonging to it is, therefore, a negligible condition in $u \in V$, which concludes the proof. 
\end{proof}

\section{A Toy Example}\label{sec:toyexmp}
Here, we provide a toy example of a polynomial CNN, where all the quantities from our results can be easily computed. We consider a network with filter sizes $k_0 = k_1 = 2$, stride $s_0 = 1$, and activation function $\sigma_2(x) = x^2$. The function computed by the network is the homogeneous quadratic polynomial:
\begin{equation}
\varphi_\mathbf{w}(x)= a^2 c x_0^2 + 2abc x_0 x_1 + (b^2 c + a^2 d)x_1^2 + 2abd x_1 x_2 + b^2 d x_2^2,
\end{equation} 
where $w_0 = (a,b), w_1 = (c,d)$. The neuromanifold has dimension $3$, degree $4$, and Euclidean distance degree $14$. The singular points, as described in Theorem 4.6, arise when $a=d=0$ or $c=b=0$. These singular points correspond to the self-intersection of the neuromanifold, which in this example is a one-dimensional line.

\end{document}